\documentclass{article}

\usepackage[main,final]{neurips_2026}

\usepackage[utf8]{inputenc} 
\usepackage[T1]{fontenc}    
\usepackage{hyperref}       
\usepackage{url}            
\usepackage{booktabs}       
\usepackage{amsfonts}       
\usepackage{nicefrac}       
\usepackage{microtype}      
\usepackage{xcolor}         
\usepackage{amsmath,amssymb}
\usepackage{amsthm}
\usepackage{thmtools}
\usepackage{wrapfig}
\usepackage{ulem}

\newtheorem{theorem}{Theorem}[section]
\newtheorem{corollary}{Corollary}[theorem]
\newtheorem{lemma}[theorem]{Lemma}
\newtheorem{result}[theorem]{Existing Result}
\newtheorem{definition}[theorem]{Definition}

\newtheorem{assumption}[theorem]{Assumption}

\newcommand{\eps}{\varepsilon}
\newcommand{\R}{\mathbb{R}}
\newcommand{\dom}{\mathsf{dom}}

\title{Optimizing Computational-Statistical Runtime for Wasserstein Distance Estimation}

\author{%
  Peter Matthew Jacobs \\
  Department of Statistics\\
  University of Wisconsin-Madison\\
  Madison, WI 53706 \\
  \texttt{pjacobs5@wisc.edu} \\
  \And
  Jeff M. Phillips \\
  Kahlert School of Computing\\
  University of Utah\\
  Salt Lake City, UT 84101\\
  \texttt{jeffp@cs.utah.edu}
}
\begin{document}

\maketitle

\begin{abstract}
  Squared Wasserstein distance is a frequently used tool to measure discrepancy between probability distributions.  This distance is typically computed between empirical measures of size $n$ from two underlying random samples. Unfortunately, even in lower dimensional Euclidean space problems $\left( d \in \{2,3\} \right)$, algorithms for Wasserstein distance computation with approximate or exact precision guarantees scale poorly in the runtime as a function of $n$ and the desired precision. In response, we consider the computational-statistical runtime, where the goal is to estimate from samples the Wasserstein distance between potentially smooth measures up to  $\eps$-additive error in expectation with respect to the sampling; we allow $O(1)$ computational  cost for collecting a sample.  Towards this, we develop a Sample-Sketch-Solve paradigm where we introduce a regular cartesian grid sketch of the samples.  We show that (especially under $\alpha$-H\"older smooth distributions) this can compress the data without increasing asymptotic error, and also regularizes the structure which enables faster exact algorithms.  Ultimately, we approximate $W_2^2(P,Q)$ within $\eps$ error in $\eps^{-\max(2,\frac{d+1+o(1)}{1+\alpha})}$ time for $0 < \alpha < 1$ H\"older smooth distributions $P,Q$ on $(0,1)^{d}$; an optimal $\Theta(\eps^{-2})$ for $\alpha > 1/2$ when $d=2$ and nearly optimal as $\alpha \to 1$ when $d = 3$.     
\end{abstract}

\section{Introduction}


In data analysis and machine learning tasks such as parameter estimation via Minimum Distance Estimation (see for example \cite{bernton2019parameter}) or non-parametric two-sample hypothesis testing (see for example \cite{ramdas2017wasserstein}), require comparison of underlying unknown probability distributions $P,Q \in \mathcal{P}(\mathbb{R}^{d})$. The squared Wasserstein distance, $W_2^2(P,Q)$, is one increasingly popular choice for this task. When $P,Q$ are supported on $\mathbb{R}^{d}$ with finite second moments it is defined as 
\[    
   W_2^2(P,Q) := \inf_{\gamma \in \Pi(P,Q)} \mathbb{E}_{(X,Y) \sim \gamma} \| X - Y \|_2^2,
\]
where $\Pi(P,Q)$ is the set of couplings of $P,Q$ (joint distributions over $\mathbb{R}^{d} \times \mathbb{R}^{d}$ with marginals $P$ and $Q$ respectively). A standard approach for estimating $W_2^2(P,Q)$ is to collect $n$ random samples from $P$ and $Q$, construct empirical measures $\hat P_n, \hat Q_n$, and use the \textit{plug-in} estimator $W_2^2(\hat{P}_n,\hat{Q}_n)$. When $n$ is large this is computationally burdensome. The best known computational algorithm can estimate $W_2^2(\hat{P}_n,\hat{Q}_n)$ (even in the non uniformly discrete case) with additive error $\eps$ in runtime $O( n \left( \frac{\Delta_\infty}{\eps} \right)^{2})$ \citep{lahn2019graph} where $\Delta_\infty$ is the maximum squared distance between pairs of points across the two point sets. 

In this paper we take the position that analysis should start with a distributional assumption on the input $P,Q$, and that it should aim for computationally efficient algorithms.  While there is value in considering algorithms for computation of Wasserstein distance between discrete point sets (e.g., $P_n,Q_n$ of size $n$), we argue that often $P_{n}$ and $Q_n$ should be considered as empirical measures from unknown distributions $P,Q$.  This means that the target accuracy of computing $W_2^2(P_n,Q_n)$ should aim to match that induced by the sample.  Finally, towards connecting this all together, we assume it takes $O(1)$ time to create one additional sample from $P$ or $Q$.  
%
%
%
The average loss of a randomized algorithm now takes into consideration both statistical error, and computational cost. A randomized algorithm in this context is a procedure $\widehat{A}$ which takes as input a precision parameter $\eps$, and randomly returns an output $\hat{A}(\eps)$ in time $O(f(\eps))$, which on average is $\eps$ close to $W_2^2(P,Q)$. We summarize this in the next definition.

\begin{definition}(Computational-Statistical Runtime (CSR))
    \label{def:CSR}
    Suppose $P$ and $Q$ are two probability distributions on $\mathbb{R}^{d}$ for which we can access random samples at $O(1)$ cost each draw.  
    We say a randomized algorithm $\{\widehat{A}(\eps) \mid \eps > 0\}$ achieves the $O(f(\eps))$ Computational-Statistical Runtime if for every $\eps >0$, $\widehat{A}(\eps)$ has runtime $O(f(\eps))$ and achieves 
    \[
    \mathbb{E} | \widehat{A}(\eps) - W_2^2(P,Q) | \leq \eps
    \]
    where the expectation is with respect to all sources of randomness induced by the algorithm.
\end{definition}


Crucially, to achieve a CSR of $O(f(\eps))$, the algorithm itself must not only choose how to use samples, but also how many of them to gather in the first place. Ultimately, this aligns with the typical challenge faced by the practitioner: choosing an algorithm to estimate the underlying quantity (in this case $W_2^2(P,Q)$) up to $\eps$ error while minimizing the time to do so, including the cost of collecting samples.  
Or alternatively, if one already has $n$ samples $P_n,Q_n$, this perspective advocates to treat them as estimates $\hat P_n, \hat Q_n$ of some unknown true distribution $P,Q$, and then the goal is to solve $W_2^2(\hat P_n, \hat Q_n)$ up to $\eps$ error, where $\eps$ is determined by $\mathbb{E} |W_2^2(P, Q) - W_2^2(\hat P_n, \hat Q_n) | = \eps$.  

As error $\eps$ scales with the diameters of the supports of $P$ and $Q$, in order to prove $\eps$-additive error, we need an assumption that $\Delta_\infty = \max_{x,x' \in \dom(P) \cup \dom(Q)} \|x-x'\|^2$ is bounded where $\dom(P)$ denotes the support of the measure $P$.  We henceforth assume the dimension $d$ is constant, and $P,Q$ are defined on $[0,1]^d$, so $\Delta_\infty = O(1)$.  
Alternatively, we could bound error in terms of $\eps \Delta_\infty$. 

\begin{wraptable}{r}{0.4\linewidth}
  \vspace{-10mm}
  \centering
  \caption{$\eps$-approximate $W_2^2$ on $[0,1]^d$}
  \label{tab:w2-runtime}
  \small
  \setlength{\tabcolsep}{2pt}
  \renewcommand{\arraystretch}{1.05}
  \begin{tabular}{@{}p{.44\linewidth}p{.15\linewidth}p{.38\linewidth}@{}}
    \toprule
    Ref. & Notes & $\tilde O(\cdot)$ Time \\
    \midrule
    \cite{kuhn1956variants} & \textsf{uw} & $n^3$ \\
    \cite{agarwal2006bipartite} & \textsf{uw,2d} & $n^{3/2}/\eps^{3/2}$ \\
    \cite{altschuler2017near} &  & $n^2/\eps^2$ \\
    \cite{dvurechensky2018computational} &  & $n^{5/2}/\eps$ \\
    \cite{lahn2019graph} &  & $n^2/\eps$+$n/\eps^2$ \\
    \cite{lahn2019graph} & \textsf{2d} & $n/\eps^2$ \\
    \cite{lahn2021n} & \textsf{uw,2d} & $n^{5/4}\,\mathrm{poly}(1/\eps)$ \\
    \bottomrule
  \end{tabular}
  \vspace{-5mm}
\end{wraptable}

\paragraph{Related Work.}
A long line of works focuses purely on the computational problem, exactly or $\eps$-approximately, computing the $W_2^2$ between discrete measures on $[0,1]^d$ of size $n$; see Table \ref{tab:w2-runtime}.   
The Hungarian algorithm \citep{kuhn1956variants} achieves $O(n^3)$ runtime in the case that the discrete measures are uniform.  The unweighted (\textsf{uw}) matching problem was improved in $\R^2$ (\textsf{2d}) to $\tilde{O}((n/\eps)^{3/2})$ by \cite{agarwal2006bipartite} and then $\tilde{O}(n^{5/4}\mathsf{poly}(1/\eps))$ by \cite{lahn2021n} (where $\tilde{O}$ denotes absorption of logarithmic factors in the asymptotic notation).  
For general $W_2^2$ computation, the Greenkhorn algorithm was originally shown to achieve $\tilde{O}( n^2/ \eps^{3})$ \citep{altschuler2017near} and then an improved analysis in \cite{lin2019efficient} yielded $\tilde{O}(n^2/ \eps^{2})$. 
\cite{dvurechensky2018computational} uses accelerated gradient descent which \cite{lin2019efficient} shows achieves $\tilde{O}(n^{5/2} / \eps)$.  \cite{lin2019efficient} separately provide an adaptive primal dual accelerated mirror descent algorithm achieving $O(n^2 \sqrt{\gamma} / \epsilon)$ where $\gamma = O(n)$, in line with the worst case bound of \cite{dvurechensky2018computational}.
\cite{lahn2019graph} then uses graph-scaling algorithms to achieve time $O(n^2/ \eps + n/ \eps^{2})$ for general Optimal Transport computation and which improves to $\tilde{O}(n/\eps^{2})$ in $\R^2$.
But these works do not address statistical accuracy for the underlying quantity $W_2^2(P,Q)$; they assume the $n$ locations signify ground truth.  

On the statistical side, the dyadic partitioning argument presented by \cite{weed2019sharp} plus triangle inequality for $W_2$ implies that when $\hat{P}_n$ and $\hat{Q}_n$ are empirical measures from $P,Q$ respectively, $\mathbb{E} |W_2^2(\hat{P}_n,\hat{Q}_n) - W_2^2(P,Q)|$ is $O(n^{-1/4})$ when $d < 4$ and and $\tilde{O}(n^{-1/d})$ for $d \geq 4$, (bounded $\Delta_\infty$ allows to convert $W_2$ closeness to $W_2^2$ closeness).  \cite{chizat2020faster} improves upon this, showing that $\mathbb{E} |W_2^2(\hat{P}_n,\hat{Q}_n) - W_2^2(P,Q)|$ is $O(n^{-1/2})$ when $d < 4$ and $\tilde{O}(n^{-2/d})$ for $d \geq 4$. 
Thus $\eps$ error needs $n = \tilde O(1/\eps^{\max(2,d/2)})$.  
Alternatives to the traditional empirical measure plug-in approach rely on Besov smoothness \citep{niles2022minimax} of $P,Q$. 

For hardness results,  \cite{niles2022estimation} imply that the number of samples needed to estimate $W_2^2(P,Q)$ up to $\eps$ in mean absolute error is $\Omega(1/\eps^{2})$. A simple, well-known (folklore) algorithm which achieves the CSR of $O(1/\eps^{2})$ in the 1 dimensional case works by collecting $n = 1/\eps^{2}$, and then sorting the sample from $P$ ($p_{(1)} < p_{(2)} < \dots p_{(n)})$ and the samples from $Q$ ($q_{(1)} < q_{(2)} < \dots < q_{(n)})$, and computing $\frac{1}{n} \sum_{j=1}^{n} (p_{(j)} - q_{(j)})^2$ (this is simply the squared Wasserstein distance between the empirical measures). The algorithm takes $O(n \log n)$ time given the samples, and the CSR is $\tilde{O}(\eps^{-2})$ (by for example \cite{chizat2020faster} Theorem 2). 
Our work will show that even for $d=2$ or $d=3$, it is possible to nearly achieve the optimal $\tilde{O}(\eps^{-2})$ CSR as long as the underlying distributions are sufficiently smooth.  


\cite{chizat2020faster} also explicitly considers the joint computational-statistical runtime. They provide CSRs using Entropic Optimal Transport, the Sinkhorn-Divergence, and a Richardson extrapolation of the Sinkhorn Divergence, respectively achieving $\tilde{O}(\eps^{-\max(6,d+2)}), \tilde{O}(\eps^{-(d'+5.5)})$ and $\tilde{O}(\eps^{-\frac{(d'+11)}{2}})$ where $d' := 2 \lfloor \frac{d}{2} \rfloor$. However, these rates are all inferior to the rate achieved by using the plug-in empirical measure approach with the $O(n \eps^{-2})$ algorithm of \cite{lahn2019graph} in $\R^2$, with the tight statistical bound for the plug-in empirical measure approach from \cite{chizat2020faster} mentioned above. These bounds are summarized in Table \ref{tab:SOTA}, along with our results which we outline next.

\begin{table}[t]
\caption{Computational-Statistical Runtimes (CSR) for $\eps$-additive $W_2^2(P,Q)$ estimation on $(0,1)^d$, where $d' = 2 \lfloor \frac{d}{2} \rfloor$.   
For conciseness we omit not only $\mathsf{poly}(\log (1/\eps))$ terms but also $\eps^{-o(1)}$ terms. Note \textbf{Our Result} [Theorem \ref{thm:main-CSR}] achieves $\eps^{-2}$ for $\alpha > 1/2$ in $\R^2$ and and $\eps^{-2}$ in $\R^3$ when $\alpha \to 1$.  }
\begin{tabular}{llccc}
\toprule
Method & Notes & CSR in $\R^d$ & in $\R^2$ & in $\R^3$\\
\midrule
Sinkhorn  \cite{chizat2020faster} &  & $\eps^{-(d'+5.5)}$ & $\eps^{-7.5}$ & $\eps^{-7.5}$  \\
SH + Richardson Extrap \cite{chizat2020faster} &  & $\eps^{-\frac{(d'+11)}{2}}$ & $\eps^{-6.5}$ & $\eps^{-6.5}$ \\
Entropic OT \cite{chizat2020faster} &  & $\eps^{-\max(6,d+2)}$ & $\eps^{-6}$ & $\eps^{-6}$  \\
Empirical Plugin + \cite{lahn2019graph} &  & $\eps^{-(\max(d,4)+1-\mathbb{I}(d=2))}$ & $\eps^{-4}$ & $\eps^{-5}$\\
\midrule
\textbf{Our Result} [Corollary \ref{cor:CSR-non-smooth}] &  & $\eps^{-\max(2,d+1)}$ & $\eps^{-3}$ & $\eps^{-4}$ \\
\textbf{Our Result} [Theorem \ref{thm:main-CSR}] & \textsf{$\alpha$-H\"ol} & $\eps^{-\max(2,\frac{d+1}{1+\alpha})}$ & $\eps^{-2}$ & $\eps^{-2}$\\
\bottomrule
\end{tabular}
\label{tab:SOTA}
\end{table}


\paragraph{Contributions.} 
We introduce what we call the \emph{Sample-Sketch-Solve} paradigm for achieving fasts CSRs for $W_2^2(P,Q)$ with $P,Q$ supported on $(0,1)^d$, and the distributions possess H\"older smooth densities.  
The paradigm operates in three phases: 
 (1) it \emph{samples} $n$ points from distributions $P$ and $Q$; 
 (2) it \emph{sketches} these samples into a compressed representation; and 
 (3) it \emph{solves} for the (approximately) optimal transportation plan on the compressed representation.  

In particular, our compressed representation is a regular Cartesian grid on $[0,1]^d$, with $L$ divisions on each dimension.  Each sample within a grid cell is mapped to the center so that there are $\max(n,L^d)$ total atoms.  We set $L = \Theta(\eps^{-\frac{1}{1+\alpha}})$ when $P$ and $Q$ are $\alpha$-H\"older smooth. This regular grid structure is useful in two ways:  for achieving a high level of precision in approximating $W_2^2(P,Q)$ when $P$ and $Q$ are smooth using the sketched measures, and for efficiently solving for the Optimal Transport cost between the sketched measures.



In Section \ref{sec:discError}, we present new bounds on the error in approximating $W_2^2(P,Q)$ between continuous probability distributions using this grid based discretization. 
We consider Holder smoothness constraints on the probability densities of $P,Q$, and rely on a novel, direct coupling based argument via a recent result of \cite{collins2025boundary} on the regularity of the Brenier map.  Still, our result here is not trivially obtainable using the triangle inequality, and is instead a result of a midpoint rule type cancellation of first order terms.  Thus we crucially rely on mapping to the center point of the bounded diameter grid cells of our sketch.  

Then in Section \ref{sec:grid-exact} we outline a new approach for efficient exact computation of $W_2^2(P,Q)$ that leverages 
(1) $P,Q$ are supported on a regular grid with $L$ splits per axis, and 
(2) each atom has normalized weight a multiple of $1/n$.  
In this setting we can use the regular grid structure and the decomposable nature of $W_2^2$ to leverage a result of \cite{auricchio2018computing} to transform the problem into a $(d+1)$-partite graph min-cost flow problem.  Then given the control on the weights and the bounded grid structure,  we control the capacities of this flow problem, and in turn apply a fast algorithm of \cite{brand2023deterministic}.  This results in a $L^{d+1+o(1)}$ exact algorithm for Wasserstein distance computation on our sketch.  

Finally, in Section \ref{sec:transportAlg} we put these pieces together, for $\alpha$-H\"older distributions, to obtain the $O(1/\eps^2)$ CSR algorithm for $\eps$-additive error for $W_2^2(P,Q)$ in $\R^2$ when $\alpha > 1/2$, nearly O($1/\eps^2)$ in $\R^3$ when $\alpha \to 1$, and in general $O(\eps^{-\max(2,\frac{d+1+o(1)}{1+\alpha})})$ in $\R^d$ for $0 < \alpha < 1$ ; we also separately analyze the general non-smooth case.  Table \ref{tab:SOTA} displays the CSRs for the general case and $\alpha$ H\"older (\textsf{$\alpha$-H\"ol}) case.

\subsection{Preliminaries}
\label{sec:prelim}
For a function $f: A \subseteq \mathbb{R}^{d} \to \mathbb{R}^{z}$, $| f |_{C^{0,\alpha}(A)} := \sup_{x,y \in A, x \neq y} \frac{\|f(x) - f(y)\|}{\| x - y\|^{\alpha}}$ denotes the $\alpha$-H\"older semi-norm of $f$ where $0 < \alpha \leq 1$. When $|f|_{C^{0,\alpha}(A)} < \infty$, we call $f$ $\alpha$-H\"older smooth~\citep{tsybakov2008nonparametric,singh2018nonparametric}. 
The higher the value of $\alpha$, the higher the level of smoothness. The $\alpha=1$ case is referred to as Lipchitz smoothness. Note that on bounded domains, by the Mean Value Theorem if $f$ is differentiable with upper bounded derivative, it is Lipchitz smooth. For $0 < \alpha \leq 1$, $f \in C^{1,\alpha}(A)$ means $f$ is differentiable and its gradient $\nabla f$ is $\alpha$-H\"older smooth. When we say a probability distribution $P$ supported on $(0,1)^{d}$ is $\alpha$-H\"{o}lder smooth for some $0 < \alpha \leq 1$, we mean it  possesses a probability density function $\mu$ which is $\alpha$-H\"{o}lder smooth on $(0,1)^{d}$. That is, for any Borel measurable subset of $(0,1)^{d}$, $A$, $P(A) = \int_{x \in A} \mu(x) dx$, and for some $C >0$, $|\mu|_{C^{0,\alpha}((0,1)^{d})} = \sup_{x,y \in (0,1)^{d}, x \neq y} \frac{|\mu(x) - \mu(y)|}{\| x - y \|^{\alpha}} = C < \infty$.  
We encapsulate these common parameterizations in the following assumption:

\begin{assumption}
    \label{assump:CoreMainText}
    $P,Q$ are probability measures on $(0,1)^{d}$ with densities $\mu$ and $\nu$ respectively satisfying for some $C >1 $ and for some $0 < \alpha <1$,
    \begin{enumerate}
    \item $P$ and $Q$ are $\alpha$-H\"older smooth:  $| \mu |_{C^{0,\alpha}((0,1)^{d})} + | \nu |_{C^{0,\alpha}((0,1)^{d})} \leq C$
    \item $P$ and $Q$ bounded above and below:  $C^{-1} \leq \mu(x),\nu(x) \leq C$ for $x \in (0,1)^{d}$
\end{enumerate}
When $P,Q$ satisfy both parts of this assumption we say they are \emph{($\alpha,C)$-controlled}.  
\end{assumption}

\paragraph{Gridding.} Consider $h > 0$ such that $L = \frac{1}{h}$ is an integer, $\mathcal{G}_{h}$ denotes the partition of $(0,1)^{d}$ constructed by partitioning each axis into intervals of width $h$. 
Let $\bar G^d_h \subset (0,1)^d$ be the $1/h^d$ points at the center of these grid cells.  
$G_{h}: (0,1)^{d} \to \bar G_h^{d}$ is the function that projects a point to the center of the grid cell in $\mathcal{G}_{h}$ containing it.

\paragraph{Pushforward.}
If $f: (0,1)^{d} \to Z$ for some set $Z$ is a Borel measurable map, and $P$ is a Borel probability measure on $(0,1)^{d}$, the notation $f \# P$ denotes the \emph{pushforward}. Specifically for measurable $A \subseteq Z$, $f \# P(A) = P(x \in (0,1)^{d}: f(x) \in A)$. $\mathrm{Id}_{A}$  refers to the identity map defined on $A$. 

\section{Grid Based Discretization Error Bounds}
\label{sec:discError}

In this section we provide an improved upper bound on the error in approximating $W_2^2(P,Q)$ by constructing discrete measures $G_{h} \# P$ and $G_{h} \# Q$ and using the plug-in approach $W_2^2(G_{h} \# P, G_{h} \# Q)$. Note that $G_{h} \# P$ (resp. $G_{h} \# Q$) are the measures constructed from $P$ (resp. $Q$) by collapsing all mass within each grid cell in each of the $O(h^{d})$ grid cells in $\mathcal{G}_{h}$ onto the center of each cell. For Entropic Optimal transport this subject was addressed by \cite{chizat2020faster} but the case with $\lambda = 0$ (i.e \textit{no} entropic regularization, which corresponds to $W_2^2$) has not been considered. We show a continuum of approximation depending on the smoothness level of the probability measures $P,Q$. 


Our proof relies on regularity results guaranteeing the smoothness of the optimal transport map which deterministically pushes forward $P$ into $Q$. We first remind the reader of the famous Theorem due to \cite{brenier1991polar} guaranteeing the existence of such a map.

\begin{result}[(Brenier) Theorem 1.16 of \cite{chewi2024statistical}]
    \label{result:brenier}
    If $P,Q$ are defined on $\mathbb{R}^{d}$ with bounded second moments, and $P$ possesses a probability density, then there exists a convex function $u :\mathbb{R}^{d} \to \mathbb{R}$ and corresponding function $Z: \mathbb{R}^{d} \to \mathbb{R}^{d} \times \mathbb{R}^{d}$, defined by $Z(x) := (x,\nabla u(x))$ such that $Z \# P$ is the unique coupling of $P$ and $Q$ achieving the infima in the definition of $W_2^2(P,Q)$.
\end{result}

By Result \ref{result:brenier}, $(\nabla u) \# P = Q$ and $T := \nabla u$ is thus referred to as the optimal transport map. Note that by a change of measure (when Brenier's Theorem applies), $W_2^2(P,Q) = \int_{x \in \mathbb{R}^{d}} \| x - T(x) \|_2^2 P(dx)$. \cite{caffarelli1992regularity,caffarelli1996boundary} initiated the study of the regularity of the optimal transport map. We use a new contribution in transport map regularity theory due to \cite{collins2025boundary} which establishes the global $\alpha$-H\"older smoothness of $\nabla u$ when $P$ and $Q$ have $\alpha$-H\"older densities themselves.

\begin{restatable}[Global Regularity of Brenier Map for Convex Domains via \cite{collins2025boundary} Theorem 5.1]{result}{collins} 
\label{res:Collins}
Suppose $P$ and $Q$ satisfy Assumption \ref{assump:CoreMainText} for some $C > 1$, $0 < \alpha < 1$. Then the convex functions $u,v$ of Brenier's Theorem satisfy $u \in C^{1,1-\xi}((0,1)^{d})$ and $v \in C^{1,1-\xi}((0,1)^{d})$ and
\[
| \nabla u |_{C^{0,1-\xi}((0,1)^{d})} + | \nabla v |_{C^{0,1-\xi}((0,1)^{d})} \leq M
\]
for any $0 < \xi < 1$ where $M$ depends only on $d,\alpha,C,\xi$. In particular the Optimal Transport maps $\nabla u$ and $\nabla v$ are $(1-\xi)$-Hölder continuous for any $0 < \xi < 1$. 
\end{restatable}

Before we proceed to the main result, we need a corollary.  
\begin{corollary}
    \label{cor:helperForDiscErrorUpperBdd}
    Let $P,Q$ be $(\alpha,C)$-controlled for some $C>1$ and $0 < \alpha < 1$, with optimal transportation map $T_\mu$ from $P$ to $Q$ (resp. $T_{\nu}$ from $Q$ to $P$), and displacement function $\ell_\mu(x) = x - T_\mu(x)$ (resp. $\ell_{\nu}(x) = x - T_\nu(x))$.  Then $\|\ell_{\mu}(x) - \ell_{\mu}(G_h(x))\| \leq K h^{1-\xi}$ (resp. $\|\ell_{\nu}(x) - \ell_{\nu}(G_h(x))\| \leq K h^{1-\xi}$) for any $0 < \xi < 1$, and where $K$ is a constant that depends only on $d$,$C, \alpha, \xi$.
\end{corollary}    
\begin{proof}
    Since $P,Q$ are $(\alpha,C)$-controlled, Existing Result \ref{res:Collins} holds. This means the OT map, $T_{\mu}$, (which is equivalently $\nabla u$ where $u$ is the convex function from Existing Result \ref{res:Collins}), is $(1- \xi)$-Hölder continuous for any $0 < \xi < 1$. Since the identity function, $f(x) = x$, is Lipchitz continuous, it is also $(1-\xi)$-Hölder continuous; also the sum of two $(1-\xi)$-Hölder continuous functions is also $(1-\xi)$-Hölder. Hence the difference between $x$ and $T_{\mu}(x)$, $\ell_{\mu}(x)$, is $(1-\xi)$-Hölder for any $\xi \in (0,1)$. The argument is analogous for $\ell_{\nu}$. 
\end{proof}

Now we state our main contribution of this section.  

\begin{restatable}{theorem}{discError}
    \label{thm:discError}
    If $P$ and $Q$ are $(\alpha,C)$-controlled (Assumption \ref{assump:CoreMainText}) for some $C > 1, 0 < \alpha < 1$, then for some constant $K$ depending only on $d,C,\alpha$, and every $h > 0$
    \[
    \left| W_2^2( G_{h} \# P, G_{h} \# Q) - W_2^2(P,Q) \right| \leq K h^{1+\alpha}
    \]
\end{restatable}

Here we show the proof of the upper bound $W_2^2(G_{h} \# P, G_{h} \# Q) - W_2^2(P,Q) \leq K h^{1+\alpha}$. The argument for the lower bound $W_2^2(G_{h} \# P, G_{h} \# Q) - W_2^2(P,Q)  \geq -K h^{1+\alpha}$ is similar (see appendix Section \ref{app:secDiscError} for complete details). In the argument below note that we repeatedly use the general constant $K$ to indicate a constant depending only on $d,C$ and $\alpha$.

\begin{proof}[proof of upper bound]
    The proof will proceed in three phases.  
    First, we will define a transportation plan $\pi$ which couples $G_h \# P$ with $G_h \# Q$ and show it is indeed a valid coupling.  
    Second, we will use this coupling to upper bound $W_2^2(G_h \#P, G_h \# Q)$ by $W_2^2(P,Q)+remainder$.  
    Third, we leverage properties of the scale of the grid cells, and the $\alpha$-H\"olderness of both $\mu,\nu$ and the Brenier maps' to bound each term in the remainder by either either $K h^2$ or $K h^{1+\alpha}$.  

     \textbf{Phase 1:} 
     Consider the probability measure $\pi:= S \# \gamma$ where $S: (0,1)^{d} \times (0,1)^{d} \to (0,1)^{d} \times (0,1)^{d}$ is defined as $S(x,y) = (G_{h}(x),G_{h}(y))$ and $\gamma$ is an optimal coupling of $P,Q$ under squared Euclidean cost. $\pi$ projects the mass of $\gamma$ onto the grid cell centers, and thus since $\gamma$ couples $P$ and $Q$, $\pi$ is a valid coupling for $G_{h} \# P_\mu$ and  $G_{h} \# P_\nu$. We now formalize this by proving that the marginals are preserved.   That is, for any $A \subseteq (0,1)^d$, 
    \begin{equation*}
    \begin{split}
    \pi(A \times (0,1)^d) &= \gamma(S^{-1}(A \times (0,1)^d)) \\
    &= \gamma((x,y) \in (0,1)^{2d} : x \in G_{h}^{-1}(A),y \in (0,1)^d) \\
    &= P(G_{h}^{-1}(A)) = G_{h} \# P(A)
    \end{split}
    \end{equation*}
    and similarly $\pi((0,1)^d \times A) = G_{h} \# Q(A)$. 
    
    \textbf{Phase 2:} 
    We can now state 
\[
        W_2^2(G_{h} \# P, G_{h} \# Q) \leq \mathbb{E}_{(X,Y)\sim \pi} \|X - Y \|_2^2 = \int_{x \in (0,1)^{d},y \in (0,1)^{d}} \| G_{h}(x) - G_{h}(y) \|_2^2 d \gamma(x,y).
\]  
       Applying Brenier's theorem we can write $\gamma = (\mathrm{Id}_{(0,1)^{d}}, T_{\mu}) \# P$ where $T_{\mu}$ is the OT map (pushing $P$ to $Q$) and so the above integral upper bound becomes (after change of measure)
\[        
    W_2^2(G_{h} \# P, G_{h} \# Q) \leq \int_{x \in (0,1)^{d}} \| G_{h}(x) - G_{h}(T_{\mu}(x)) \|_2^2 \mu(x) dx.
\]
    To isolate the contribution of $W_2^2(P,Q)$ to the integral on RHS, we denote $\Delta z = z - G_{h}(z)$ for $z \in (0,1)^{d}$ so that
    \[
    G_{h}(x) - G_{h}(T_{\mu}(x)) = \left(x - T_{\mu}(x) \right) - (\Delta x  - \Delta T_{\mu}(x)).
    \]
    Thus using Brenier's Theorem and expanding the square in the integral using the above identity, and using that 
    $W_2^2(P,Q) = \int_x (x - T_\mu(x))^2 \mu(x) dx$, 
    we have
    \[
    W_2^2(G_{h} \# P, G_{h} \# Q) - W_2^2(P,Q) \leq \int_{x} \| \Delta x - \Delta T_{\mu}(x) \|_2^2 - 2 \langle x - T_{\mu}(x), \Delta x - \Delta T_{\mu}(x) \rangle \mu(x) dx.
    \]

    \textbf{Phase 3:}     We will use three properties:
\begin{itemize}
  \item[\textbf{[P1]}] For all $x \in (0,1)^{d}$, $\| \Delta x \|_2 = \|x - G_{h}(x)\| \leq \sqrt{d} h$ .
  \item[\textbf{[P2]}] The Brenier displacement $\ell_{\mu}(x) := x - T_{\mu}(x)$ takes values in the bounded domain $(-1,1)^d$.
  \item[\textbf{[P3]}] Since $P,Q$ are $(\alpha,C)$-controlled, by Corollary \ref{cor:helperForDiscErrorUpperBdd}, the Brenier displacement $\ell_{\mu}(x) = x - T_{\mu}(x)$ satisfies $\|\ell(x) - \ell(G_h(x))\| \leq K h^{\alpha}$ (by setting $\xi = 1-\alpha$). 
%
  \item[\textbf{[P4]}] Since $P,Q$ are $(\alpha,C)$-controlled, $\mu(x) \leq C = O(1)$. Additionally, as $\mu$ is $\alpha$-Hölder, $|\mu(x) -\mu(G_h(x))|\leq K h^{\alpha}$.  
\end{itemize}

     The first term in the remainder, $\int_x \|\Delta x - \Delta T_\mu(x)\|_2^2 dx$, scales as $K h^{2}$ by triangle inequality yielding $\|\Delta x - \Delta T_\mu(x)\| \leq \|\Delta x\| + \|\Delta T_\mu(x)\|$ and then application of [\textbf{P1}].  We are left with
    \begin{equation}
        \label{discretizationErrorUpperPenultimate}
    W_2^2(G_{h} \# P_{\mu}, G_{h} \# P_{\nu}) - W_2^2(P_{\mu},P_{\nu}) \leq K h^{2} - 2 \int_{x \in (0,1)^{d}} \langle x - T_{\mu}(x), \Delta x - \Delta T_{\mu}(x) \rangle \mu(x) dx.
    \end{equation}
    It remains to bound the absolute value of the remaining integral. We have
\[
\begin{aligned}
    |\int_{x} \langle x - T_{\mu}(x), \Delta x - \Delta T_{\mu}(x) \rangle \mu(x) dx| 
    &\leq \\
    |\int_x \langle x - T_{\mu}(x), \Delta x \rangle \mu(x) dx| + | \int_x \langle x - T_{\mu}(x), \Delta T_{\mu}(x) \rangle \mu(x) dx|.
\end{aligned}
\]
    The second of these terms we can show is symmetric to the first by using from Brenier's Theorem that the optimal coupling is unique.  In particular, we can re-write $\gamma = (\mathrm{Id}_{(0,1)^d},T_\mu) \# P$ from above in the other direction as $\gamma = (T_{\nu},\mathrm{Id}_{(0,1)^{d}}) \# Q$ (where $T_{\nu}$ is the OT map pushing $Q$ to $P$).  Then we can rewrite this by first integrating over the coupling measure, and then swapping the direction:
    \[
\begin{aligned}
    \int_{x \in (0,1)^{d}} \langle x - T_{\mu}(x), \Delta T_{\mu}(x) \rangle \mu(x) dx &= \int_{(x,y) \in (0,1)^{d} \times (0,1)^{d}} \langle x - y, \Delta y \rangle d \gamma(x,y) \\
      &= \int_{y \in (0,1)^{d}} \langle T_{\nu}(y) - y, \Delta y \rangle \nu(y) dy.
\end{aligned}
\]
    Now the two terms we need to bound 
    $| \int_{x \in (0,1)^{d}} \langle x - T_{\mu}(x), \Delta x \rangle \mu(x) dx|$  and 
    $|\int_{y \in (0,1)^{d}} \langle T_{\nu}(y) - y, \Delta y \rangle \nu(y) dy|$ 
    are symmetric to each other, and just show how to bound the first term.

    We arbitrarily label the cells $C_1,\dots,C_{(1/h)^{d}}$ (without loss of generality assuming $1/h$ is an integer), and the corresponding cell centers as $c_{1},\dots,c_{(1/h)^{d}}$.  
    Observe that for every $j$, $\int_{C_j} \langle \ell_{\mu}(c_j),x-c_j \rangle \mu(c_j) = 0$ since $C_j$ is the grid cell centered at $c_j$. 
    We can now rewrite this first term, and subtract this $0$ expression:
\begin{align*}
& | \int_{x \in (0,1)^{d}} \langle x - T_{\mu}(x), \Delta x \rangle \mu(x)| dx &
\\ & =  
\left| \int_{x \in (0,1)^{d}} \langle \ell_{\mu}(x), \Delta x \rangle \mu(x) dx - \int_{x \in (0,1)^{d}} \sum_{j} \mathbb{I}(x \in C_j) \langle \ell_{\mu}(c_j),x-c_j \rangle \mu(c_j) dx \right| 
 \\
 & \leq 
 \sum_{j}\int_{x \in C_j} \left| \langle \ell_{\mu}(x),x-c_j \rangle \mu(x) - \langle \ell_{\mu}(c_j),x-c_j \rangle \mu(c_j)\right| dx  
 & \tag{by Tri Ineq}
 \\
& \leq \sum_{j} \int_{x \in C_j} | \langle \ell_{\mu}(x), x - c_j \rangle \mu(x) - \langle \ell_{\mu}(x), x-c_j \rangle \mu(c_j) + \langle \ell_{\mu}(x), x - c_j \rangle \mu(c_j) - \langle \ell_{\mu}(c_j),x-c_j \rangle \mu(c_j)| dx
& \tag*{by $\pm \langle \ell_{\mu}(x), x-c_j \rangle \mu(c_j)$}
\end{align*}
\begin{align*}
& \leq \sum_{j} \int_{C_j}| \langle \ell_{\mu}(x),x-c_j \rangle (\mu(x) - \mu(c_j))| dx + \int_{C_j} |\langle \ell_{\mu}(x) - \ell_{\mu}(c_j),x-c_j \rangle \mu(c_j)| dx  
& \tag*{by Tri Ineq}
\\
& \leq
\sum_{j} \int_{C_j} \| \ell_{\mu}(x) \|_2 \| x - c_j \|_2 |\mu(x) - \mu(c_j)|dx + \int_{C_j} \| \ell_{\mu}(x) - \ell_{\mu}(c_j) \|_2 \| x-c_j \|_2 |\mu(c_j)| dx
 & \tag*{by Cauchy-Schwarz}
\\
& \leq K 
\left(\sum_{j} \int_{C_j} h |\mu(x) - \mu(c_j)| dx + \int_{C_j} h \| \ell_{\mu}(x) - \ell_{\mu}(c_j) \| dx \right) &
\tag*{ by [\textbf{P2}] and [\textbf{P4}]}
\\
& \leq K \sum_{j} Vol(C_j) h^{1+\alpha} 
= K h^{1+\alpha} &
\tag*{by [\textbf{P3}] and [\textbf{P4}]}
\end{align*}
    
    Thus $| \int_{x \in (0,1)^{d}} \langle x - T_{\mu}(x), \Delta(x) \rangle \mu(x) dx| \leq K h^{1+\alpha}$. An analagous argument yields $|\int_{y \in (0,1)^{d}} \langle T_{\nu}(y) - y, \Delta y \rangle \nu(y) dy| \leq K h^{1+\alpha}$. The upper bound $W_2^2(G_{h} \# P, G_{h} \# Q) - W_2^2(P,Q) \leq K h^{1+\alpha}$ follows from these and Equation \ref{discretizationErrorUpperPenultimate}.
\end{proof}

Just by using that $|x-y| \leq \frac{|x^2 - y^2|}{x+y}$, we additionally have the following Lemma.
\begin{lemma}
    \label{lem:discError}
    Assuming $P,Q$ are $(\alpha,C)$-controlled for some $C > 1$ and $0 < \alpha < 1$ and $W_2(P,Q) > \eps > 0$ then for some constant $K$ depending only on $d,c,\alpha$ and every $h > 0$,
    \[
    \left| W_2( G_{h} \# P, G_{h} \# Q) - W_2(P,Q) \right| \leq \frac{K h^{1+\alpha}}{\eps}
    \]
\end{lemma}
\begin{proof}
Using Theorem \ref{thm:discError}
    \[
    |W_2(G_{h} \# P,G_{h} \# Q) - W_2(P,Q) | \leq \frac{|W_2^2(G_{h} \# P,G_{h} \# Q) - W_2^2(P,Q)|}{W_2(G_{h} \# P,G_{h} \# Q) + W_2(P,Q)} \leq \frac{K h^{1+\alpha}}{\eps}
    \]
\end{proof}

\subsection{Non-Smooth Case}
Note that the $\alpha \to 0$ bound of $O(h)$ holds for any distributions $P$ and $Q$, regardless of their level of smoothness. We prove this below.
\begin{lemma} \label{lem:non-smooth}
    If $P$ and $Q$ are probability measures on $(0,1)^{d}$, then for $h > 0$
    \[
    |W_2(G_{h} \# P, G_{h} \# Q) - W_2(P,Q)| \leq h \left( 2 \sqrt{d}\right).
    \]
    and
    \[
    |W_2^2(G_h \# P, G_h \# Q) - W_2^2(P,Q)| \leq h \left(4 d^{3/2} \right).
    \]
\end{lemma}
\begin{proof}
    Let $\pi_{P} = (\mathrm{Id}_{(0,1)^{d}},G_{h})\#P$ and $\pi_{Q} = (\mathrm{Id}_{(0,1)^{d}},G_{h})\# Q$. Then $\pi_{P}$ is a coupling of $P$ and $G_{h}  \# P$ and $\pi_{Q}$ is a coupling of $Q$ and $G_{h} \# Q$. Using a change of measure we thus have $W_2^2(G_{h} \# P , P) \leq \int_{x \in (0,1)^{d}} \| x - G_{h}(x) \|_2^2 P(dx) = \sum_{j} \int_{x \in C_j} \| x - c_j \|_2^2 P(dx) \leq d h^2$, and likewise $W_2^2(G_{h} \# Q, Q) \leq d h^2$. Using these and triangle inequality of $W_2$, we conclude that 
    \begin{equation}
        \label{eqn:wasserRes}
    |W_2(G_{h} \# P, G_{h} \# Q) - W_2(P,Q)| \leq W_2(G_h\#P,P)+W_2(G_h\#Q,Q) \leq h \left( 2 \sqrt{d} \right).
    \end{equation}
    Now using that $P,Q$ are supported on $(0,1)^{d}$ (and thus so are $G_h \# P$ an $G_h \# Q$), we also have that $\max(W_2^2(P,Q),W_2^2(G_h\#P,G_h\#Q)) \leq d$. Using these and Equation \ref{eqn:wasserRes} and that for $x,y \in \mathbb{R}$, $|x^2-y^2| = |x-y| |x+y|$, we conclude that
    \[
    |W_2^2(G_{h} \# P, G_{h} \# Q) - W_2^2(P,Q)| \leq h \left(2 \sqrt{d} \right) 2 d = h \left( 4 d^{3/2} \right).
    \]
\end{proof}

It is not hard to show that $W_2(G_{h} \# P, P) = \Theta(h)$ when $P$ is for example the uniform distribution on $(0,1)^{d}$.   
And thus the brute force best bound one can obtain using the triangle inequality for $W_2$ is $\left| W_2( G_{h} \# P, G_{h} \# Q) - W_2(P,Q) \right| = O(h).$ 


\section{Efficient Exact Computation on a Regular Grid}
\label{sec:grid-exact}

In this section we isolate a structured regime in which the squared Wasserstein distance admits a substantially faster \emph{exact} algorithm. Let $L,k \in \mathbb{N}$, and recall that $\bar G_{1/L}^d$ are the centers of a regular grid on $(0,1)^d$.  
Let
\[
P = \sum_{x \in \bar G^d_{1/L}} p_x \, \delta_x,
\qquad
Q = \sum_{x \in \bar G^d_{1/L}} q_x \, \delta_x
\]
be probability measures supported on $\bar G^d_{1/L}$, with $p_x,q_x \in \frac{1}{k}\mathbb{Z}_{\ge 0}$ for every grid point $x$. After aggregating repeated support points, the input can be viewed as a pair of $d$-dimensional histograms on
regular $L^d$ axis-aligned bins of $(0,1)^d$.

The squared Euclidean ground cost is separable on this grid. Indeed, if
\[
x = \left(\frac{2a_1+1}{2L},\ldots,\frac{2a_d+1}{2L}\right),
\qquad
y = \left(\frac{2b_1+1}{2L},\ldots,\frac{2b_d+1}{2L}\right),
\]
with \(a_1,\ldots,a_d,b_1,\ldots,b_d \in \{0,1,\ldots,L-1\}\), then
\[
\|x-y\|_2^2 = \frac{1}{L^2}\sum_{\ell=1}^d (a_\ell-b_\ell)^2.
\]
Therefore the cost has the separable form required by the exact reduction of \citet{auricchio2018computing}, which maps optimal transport between $d$-dimensional histograms with separable ground cost to an equivalent uncapacitated minimum-cost flow problem on a $(d+1)$-partite graph. Instantiating their construction with $L^d$ histogram bins yields a graph with $(d+1)L^d$ vertices and
$m := dL^{d+1}$ arcs.

\begin{theorem}\label{thm:grid-exact-general-d}
Assume $L,k \in \mathbb{N}$. Let $P$ and $Q$ be probability measures supported on the cell-center grid $\bar G^d_{1/L} \subset [0,1]^d$
and assume every mass is an integer multiple of $1/k$. Then $W^2_2(P,Q)$, for every fixed dimension $d$, can be computed exactly in time
\[
(dL^{d+1})^{1+o(1)} \log L \, \log k. 
\]
For for every fixed dimension $d$, this is $\tilde{O}(L^{d+1+o(1)})$ where $\tilde{O}(\cdot)$ hides logarithmic factors of $L,k$.  
\end{theorem}

\begin{proof}
By the exact reduction of \citet{auricchio2018computing}, it suffices to solve the resulting uncapacitated minimum-cost flow instance on a $(d+1)$-partite graph with $m = dL^{d+1}$ arcs. In this reduction, each arc cost is one coordinate contribution to the squared Euclidean cost, hence is of the form
$
\frac{(a_\ell-b_\ell)^2}{L^2}
$
for some coordinate index $\ell$ and some integers $a_\ell,b_\ell \in \{0,1,\ldots,L-1\}$. Therefore every arc cost is a multiple of $1/L^2$.  

Next scale every supply and demand by $k$. Because each histogram mass lies in $\frac{1}{k} \mathbb{Z}_{\ge 0}$, all node imbalances become integral; moreover, since $P$ and $Q$ are probability measures, the total amount of flow becomes exactly $k$. 
Now also scale every arc cost by $L^2$. After this rescaling, all costs become integral, and every cost is at most
\[
\max_{0 \le a_\ell,b_\ell \le L-1} (a_\ell-b_\ell)^2 \le (L-1)^2 < L^2.
\]
Since the total flow is $k$, assigning capacity $k$ to every formerly uncapacitated arc does not change the feasible region. 
Thus we obtain a directed minimum-cost flow instance with integral demands, capacities, and costs, where the maximum edge capacity satisfies $U \le k$ and the maximum edge cost satisfies $C < L^2$.

We may therefore apply the deterministic exact minimum-cost flow algorithm of \cite{brand2023deterministic}, which computes an exact optimum in time 
$
m^{1+o(1)} \log U \log C,
$
for instances with integral demands, capacities, and costs.  
Substituting $U \leq k$ and $C < L^2$ yields a running time of 
\[
m^{1+o(1)} \log k \log L.
\]


Finally, let $\mathrm{OPT}$ denote the original value of $W^2_2(P,Q)$, and let $\widehat{\mathrm{OPT}}$ denote the optimum of the scaled min-cost flow instance. Scaling the flow values by $k$ multiplies the objective by $k$, and scaling the arc costs by an additional factor of $L^2$ multiplies the objective by another factor of $L^2$. Hence
\[
\widehat{\mathrm{OPT}} = k L^2 \mathrm{OPT}.
\]
Dividing by $kL^2$ therefore recovers the exact value of $W^2_2(P,Q)$. 
\end{proof}

\begin{corollary}\label{cor:grid-exact-2d}
Under assumptions of Theorem~\ref{thm:grid-exact-general-d}, when $k = \mathsf{poly}(L)$ computing $W^2_2(P,Q)$ 
\begin{itemize}  
    \item in $d=2$ takes time $L^{3 + o(1)}$
    \item in $d=3$ takes time $L^{4 + o(1)}$.  
\end{itemize}
\end{corollary}

\section{Faster CSR Algorithm for $W_2^2$}
\label{sec:transportAlg}


In this section we present a fast \emph{Computational-Statistical Runtime (CSR)} algorithm for additive $\eps$ approximation of the $W_2^2$ between two measures $P,Q$ supported on $(0,1)^{2}$.  This combines our two above results:
 first applying the statistical result about Grid-based discretization (Theorem \ref{thm:discError}) to compress a sample, and
 second applying the exact algorithm when our samples have been discretized to a grid (Theorem \ref{thm:grid-exact-general-d}) for computing $W_2^2$.

\begin{theorem}\label{thm:main-CSR}
    Let $P$ and $Q$ be $(\alpha,C)$-controlled distributions (with $C>1, 0 < \alpha < 1$) on $(0,1)^d$, for which we can draw random samples in $O(1)$ time each.  For constant dimension $d$ and $\eps > 0$, we can compute $\hat A(\eps)$ so that in expectation $|\hat A(\eps) - W_2^2(P,Q)| \leq \eps$, in time 
    \[
    O((1/\eps)^{\max \left(2,\frac{d+1+o(1)}{1+\alpha}\right)}).
    \]
\end{theorem}
\begin{proof}
    The proof follows the three phases of the sample-sketch-solve paradigm.  
    We start by sampling $n = O(1/\eps^{\max(2,d/2)})$ samples from both $P$ and $Q$ to produce discrete measures $\hat P_n, \hat Q_n$, respectively. Next we apply a grid-based discretization with an edge length $h = O(\eps^{\frac{1}{1+\alpha}})$ to obtain $G_h \# \hat P_n$ and $G_h \# \hat Q_n$. By Theorem \ref{thm:discError} we have $|W^2_2(P,Q) - W^2_2(G_h \# P,G_h \# Q)| \leq K h^{1+\alpha} = \eps/2$. Since $G_{h} \# \hat{P}_n$ is an empirical measure of $n$ independent samples of $G_{h} \# P$ and $G_{h} \# \hat{Q}_{n}$ is an empirical measure of $n$ independent samples of $G_{h} \# Q$, by \cite{chizat2020faster} Theorem 2 we have $\mathbb{E} |W^2_2(G_h \# P,G_h \# Q) - W^2_2(G_h \# \hat P_n,G_h \# \hat Q_n)| \leq \eps/2$ where the expectation is with respect to the sampling from $P$ and $Q$. Thus with $\hat{A}(\eps) := W_2^2(G_{h} \# \hat{P}_n, G_h \# \hat{Q}_n)$ and applying triangle inequality, 
    \begin{align*}
    \mathbb{E} |W_2^2(P,Q) - \hat{A}(\eps) | &
    =
    \\
    \mathbb{E} |W^2_2(P,Q) - W^2_2(G_h \# \hat P_n,G_h \# \hat Q_n)| & \leq 
    \\ 
    |W^2_2(P,Q) - W^2_2(G_h \# P,G_h \# Q)| + \mathbb{E}|W^2_2(G_h \# P,G_h \# Q) - W^2_2(G_h \# \hat P_n,G_h \# \hat Q_n)| 
    & \leq \eps.
    \end{align*}
    Finally, now the distributions $G_h \# \hat P_n,G_h \# \hat Q_n$ lie on a regular Cartesian grid $\bar G^d_{1/L}$ with $L = 1/h =  O(1/\eps^{1/(1+\alpha)})$ and each atom has mass that is an integer multiple of $1/n$ with $n = O(1/\eps^{\max(2,d/2)})$.  Thus we can invoke Theorem \ref{thm:grid-exact-general-d} to exactly compute $W_2^2(G_h \# \hat P_n,G_h \# \hat Q_n)$ in $L^{d+1 + o(1)}\log n = 1/\eps^{(d+1+o(1))/(1+\alpha)} $ time.  Adding in the $\tilde{O}(1/\eps^{\max(2,d/2)})$ time for $n$ samples completes the proof.  
\end{proof}
We now provide the following corollaries about how this analysis is exactly or nearly optimal for $d=2,3$ dimensions. 

\begin{corollary}
    Let $P$ and $Q$ be $(\alpha,C)$-controlled distributions on $(0,1)^d$ for $d=2$, for which we can draw random samples in $O(1)$ time each.  For $\alpha > 1/2$ and $\eps > 0$, we can compute $\hat A(\eps)$ so that in expectation $|\hat A(\eps) - W_2^2(P,Q)| \leq \eps$, in time $O(1/\eps^2)$.  
\end{corollary}

\begin{corollary}
    Let $P$ and $Q$ be $(\alpha,C)$-controlled  distributions on $(0,1)^d$ for $d=3$, for which we can draw random samples in $O(1)$ time each.  For $\eps > 0$, we can compute $\hat A(\eps)$ so that in expectation $|\hat A(\eps) - W_2^2(P,Q)| \leq \eps$, in time $\eps^{-\frac{4+o(1)}{1+\alpha}}$, which approaches  $\eps^{-(2+o(1))}$ as $\alpha \to 1$.
\end{corollary}

\subsection{Non-Smooth Case} 
When we do not have a bound on the smoothness of $P,Q$, then we can use Lemma \ref{lem:non-smooth}, and as a result Theorem \ref{thm:main-CSR} behaves as with $\alpha=0$ with a CSR of $\eps^{-\max(2,d+1+o(1))}$.  

\begin{corollary}
\label{cor:CSR-non-smooth}
    Let $P$ and $Q$ be distributions on $(0,1)^d$, for which we can draw random samples in $O(1)$ time each.  For constant dimension $d$ and $\eps > 0$, we can compute $\hat A (\eps)$ so that in expectation $|\hat A (\eps) - W_2^2(P,Q)| \leq \eps$, in time 
    \[
    \tilde{O}(\eps^{-\max\!\big(2, \; d+1+o(1) \big)}).
    \]  
    For $d=2$ the CSR is $1/\eps^{3+o(1)}$ and for $d=3$ it is $1/\eps^{4+o(1)}$.  
\end{corollary}

\section{Limitations}

In the solve step of the Sample-Sketch-Solve paradigm to achieve fast theoretical CSRs, we proposed usage of the deterministic exact minimum-cost flow algorithm of \cite{brand2023deterministic}. This algorithm has large hidden constants in the runtime stated in Section \ref{sec:grid-exact}, and development of a practically fast implementation is still an open area of research.

\section{Acknowledgments and Funding Disclosure}

The authors are thankful for support for this project provided from NSF grants 2115677, 2311954, 2421782, DOE grant DE-SC0021015, and Simons Foundation grant MPS-AI-00010515. Additionally, the authors thank the anonymous reviewers for their helpful feedback.

\clearpage

\bibliographystyle{plainnat}
\bibliography{references_arxiv}






\clearpage
\appendix



\section{Discretization Error for $W_2^2$}
\label{app:secDiscError}

In this section we prove the following theorem:

\discError*
In the main body of the paper, we already proved the upper bound 
\[W_2^2(G_h \# P, G_h \# Q) - W_2^2(P,Q) \leq K h^{1+\alpha}. \]
It remains to prove the corresponding lower bound 
\[
W_2^2(G_h \# P, G_h \# Q) - W_2^2(P,Q) \geq -K h^{1+\alpha},
\] 
and we take this up in the next Subsection.

\subsection{Proof of Discretization Error Lower Bound}
\label{sec:lowerBoundProof}

As we did in the upper bound proof, we continue to assume without loss of generality that $\frac{1}{h}$ is an integer.

\begin{lemma}[Discretization Error Lower Bound]
    \label{thm:discretizationErrorLowerBound}
    If $P$ and $Q$ are $(\alpha,C)$ controlled (Assumption \ref{assump:CoreMainText}) for some $C > 1$, $0 < \alpha < 1$, then for some constant $K$ depending only on $d,C,\alpha$, and every $h > 0$,
    \[
    W_2^2(P,Q) \leq W_2^2(G_{h} \# P, G_{h} \# Q) + K h^{1+\alpha}.
    \]
\end{lemma}
Note that the below argument only relies on smoothness of the densities, not at all on the transport maps. Existing Result \ref{res:Collins}, which guarantees the smoothness of the transport maps, uses that the densities are bounded below. Since we do not use transport map smoothness in the lower bound argument, we do not need the densities to be bounded below. In particular, only for the sake of consistency with the assumption of the upper bound does Theorem \ref{thm:discretizationErrorLowerBound} say $P,Q$ are $(\alpha,C)$ controlled. This is stronger than what is nescessary for the lower bound argument. The lower bound argument uses only the smoothness of the density $\mu$ associated to $P$ and $\nu$ associated to $Q$.

This argument has a similar flavor to the upper bound argument. We find a coupling for $P$ and $Q$, call it $\Gamma$, derived from an optimal coupling $\pi^{*}$ for $G_{h} P$ and $G_{h} Q$. Then we extract the cost $W_2^2(G_{h} \# P, G_{h} \# Q)$ from the cost of $\Gamma$, and upper bound the leftover quantity utilizing the smoothness of the densities.
\begin{proof}
    Let $\pi^{*}$ be an optimal coupling between the discrete measures $G_{h} \# P$ and $G_{h} \# Q$. 
    We define a probability measure $\Gamma$ on $(0,1)^{2d}$ by defining the probability density function $\gamma :(0,1)^{2d} \to \mathbb{R}$ as
    \begin{equation}
        \label{defOfGamma}
        \gamma(x,y) = \sum_{i} \sum_{j} \mathbb{I}(x \in C_i) \mathbb{I}(y \in C_j) \frac{\pi^{*} \left( c_i,c_j\right)}{P(C_i) Q(C_j)} \mu(x) \nu(y).
    \end{equation}
    and setting, for $A \subseteq (0,1)^{2d}$, $\Gamma(A) = \int_{(x,y) \in A} \gamma(x,y) d(x,y)$. Within $C_i \times C_j$, $\Gamma$ retains the shape of the independent coupling $P \times Q$, but rescales the measure to be the amount of measure transferred between the center of $C_i$ and $C_j$ under the optimal coupling between the measures snapped onto the grid cell centers. We now show $\Gamma$ is indeed a coupling of $P$ and $Q$, and then we extract the quantity $W_2^2(G_{h} \# P, G_{h} \# Q)$ from the cost of $\Gamma$. But first, note that 
    \begin{equation}
        \begin{split}
            \int_{(x,y) \in (0,1)^{2d}} \gamma(x,y) d(x,y) &= \int_{x \in (0,1)^d} \int_{y \in (0,1)^d} \gamma(x,y) dy dx = \\
            &= \sum_{i} \sum_{j} \int_{x \in C_i} \int_{y \in C_j} \frac{\pi^{*}(c_i,c_j)}{P(C_i) Q(C_j)} \mu(x) \nu(y) dy dx \\
            &= \sum_{i} \sum_{j} \pi^{*}(c_i,c_j) \\
            &= 1.
        \end{split}
    \end{equation}
    where in the last inequality we used that $\pi^{*}$ is a probabilty measure supported on $(c_i)_{i \in h^{d}} \times (c_j)_{j \in h^{d}}$. 
    Thus $\gamma$ is indeed a probability density on $(0,1)^{2d}$. Additionally, for $A \subseteq (0,1)^{d}$, we have that
    \begin{equation}
        \begin{split}
            \Gamma(A \times (0,1)^{d}) &= \sum_{i} \Gamma(\left(A \cap C_i\right) \times (0,1)^{d}) \\
            &= \sum_{i} \sum_{j} \int_{x \in A \cap C_i} \int_{y \in C_j} \frac{\pi^{*}(c_i,c_j)}{P(C_i) Q(C_j)} \mu(x) \nu(y) dy dx \\
            &= \sum_{i} \sum_{j} \frac{\pi^{*}(c_i,c_j)}{P(C_i) Q(C_j)} P(A \cap C_i) Q( C_j) \\
            &= \sum_{i} \frac{P(A \cap C_i)}{P(C_i)} \sum_{j} \pi^{*}(c_i,c_j) \\
            &= \sum_{i} \frac{P(A \cap C_i)}{P(C_i)} G_{h} \# P(c_i) \\
            &= \sum_{i} \frac{P(A \cap C_i)}{P(C_i)} P(C_i) \\
            &= P(A).
        \end{split}
    \end{equation}
    where we used that $\pi^{*}$ is a coupling between $G_{h} \# P$ and $G_{h} \# Q$ and then the definition of $G_{h} \# P$. By a symmetric argument, for $A \subseteq (0,1)^{d}$,
    \[
    \Gamma((0,1)^{d} \times A) = Q(A).
    \]
    In particular $\Gamma$ is a coupling of $P$ and $Q$. Thus we have that
    \begin{equation}
        \label{lowerBoundPart1}
        \begin{split}
        W_2^2(P,Q) \leq \int_{(x,y) \in (0,1)^{2d}} \|x-y\|_2^2 \gamma(x,y) d(x,y) = & \\
        \sum_{i} \sum_{j} \int_{x \in C_i} \int_{y \in C_j}  \frac{\pi^{*}(c_i,c_j)}{P(C_i) Q(C_j)} \|x-y\|_2^2 \mu(x) \nu(y) dy dx
        \end{split}
    \end{equation}
    Now we extract the cost $W_2^2(G_{h} \# P, G_{h} \# Q)$ from the cost of $\Gamma$. Using that for each $i \in [(1/h)^{d}], j \in [(1/h)^{d}]$, $x,y \in (0,1)^{d}$, $x-y = (x - c_i) + (c_j - y) + (c_i - c_j)$, we have that
    \begin{equation}
        \begin{split}
            \| x - y\|_2^2 = \| x - c_i \|_2^2 + \|c_j - y\|_2^2 + \| c_i - c_j \|_2^2 + 2 \langle x - c_i, c_j - y \rangle + 2 \langle x - c_i, c_i - c_j \rangle + & \\
            2 \langle c_j - y, c_i - c_j \rangle.
        \end{split}
    \end{equation}
    Expanding $\|x - y \|_2^2$ in this way in equation \ref{lowerBoundPart1}, we have that
    \begin{equation}
        \label{lowerBoundPart2}
        \begin{split}
           W_2^2(P,Q) \leq \sum_{i} \sum_{j} \int_{x \in C_i} \frac{\pi^{*}(c_i,c_j)}{P(C_i)} \| x - c_i \|_2^2 \mu(x) dx + & \\
           \sum_{i} \sum_{j} \int_{y \in C_j} \frac{\pi^{*}(c_i,c_j)}{Q(C_j)}\| y - c_j \|_2^2 \nu(y) dy + & \\
           \sum_{i} \sum_{j} \pi^{*}(c_i,c_j) \| c_i - c_j \|_2^2  + & \\
           2 \sum_{i} \sum_{j} \frac{\pi^{*}(c_i,c_j)}{P(C_i) Q(C_j)} \int_{x \in C_i} \int_{y \in C_j} \left( \langle x - c_i, c_j - y \rangle + \langle x - c_i, c_i - c_j \rangle + \langle c_j - y, c_i - c_j \rangle \right) \mu(x) \nu(y) dy dx.
        \end{split}
    \end{equation}
    $\| x - c_i \|_2^2 \leq K h^2$ for $x \in C_i$ and likewise $\| y - c_j \|_2^2 \leq K h^2$ for $y \in C_j$ for general constant $K$ depending only on $d$. 
    Thus the total of the first two terms on the RHS of equation \ref{lowerBoundPart2} is bounded by $K h^2$. Since $\pi^{*}$ is an optimal coupling between $G_{h} \# P$ and $G_{h} \# Q$, the third
    term on RHS of equation \ref{lowerBoundPart2} is exactly $W_2^2(G_{h} \# P, G_{h} \# Q)$. Thus we have that
    \begin{equation}
        \label{lowerBoundPart3}
        \begin{split}
            W_2^2(P,Q) \leq W_2^2(G_{h} \# P, G_{h} \# Q) + K h^2 + \\
            2 \sum_{i} \sum_{j} \frac{\pi^{*}(c_i,c_j)}{P(C_i) Q(C_j)} \int_{x \in C_i} \int_{y \in C_j} \left( \langle x - c_i, c_j - y \rangle + \langle x - c_i, c_i - c_j \rangle + \langle c_j - y, c_i - c_j \rangle \right) \mu(x) \nu(y) dy dx.
        \end{split}
    \end{equation}
    By Cauchy-Schwarz, and again using that the diameter of $C_i,C_j$ is at most $K h$, we have that
    \begin{equation}
        \label{productTerm}
        \begin{split}
            \left| 2 \sum_{i} \sum_{j} \frac{\pi^{*}(c_i,c_j)}{P(C_i) Q(C_j)} \int_{x \in C_i} \int_{y \in C_j} \left( \langle x - c_i, c_j - y \rangle \right) \right| \leq & \\
            2 \sum_{i} \sum_{j} \frac{\pi^{*}(c_i,c_j)}{P(C_i) Q(C_j)} \int_{x \in C_i} \int_{y \in C_j} \| x - c_i \|_2 \| c_j - y \|_2 \mu(x) \nu(y) dy dx \leq & \\
            2 \sum_{i} \sum_{j} \frac{\pi^{*}(c_i,c_j)}{P(C_i) Q(C_j)} \left(\int_{x \in C_i} \| x - c_i \|_2 \mu(x) dx \right) \left(\int_{y \in C_j} \| c_j - y \|_2  \nu(y) dy \right) \leq & \\
            K h^2 \sum_{i} \sum_{j} \pi^{*}(c_i,c_j) \leq & \\
            K h^2.
        \end{split}
    \end{equation}
    By Equations \ref{lowerBoundPart3} and \ref{productTerm}, we have that
    \begin{equation}
        \label{lowerBoundPart4}
        \begin{split}
            W_2^2(P,Q) \leq W_2^2(G_{h} \# P, G_{h} \# Q) + K h^2 + & \\
            2 \sum_{i} \sum_{j} \frac{\pi^{*}(c_i,c_j)}{P(C_i) Q(C_j)} \int_{x \in C_i} \int_{y \in C_j} \left( \langle x - c_i, c_i - c_j \rangle + \langle c_j - y, c_i - c_j \rangle \right) \mu(x) \nu(y) dy dx.
        \end{split}
    \end{equation}
    To handle the first remaining integrand, we express $\mu(x) = \mu(x) - \mu(c_i) + \mu(c_i)$. Then we have that
    \begin{equation}
        \begin{split}
           2 \sum_{i} \sum_{j} \frac{\pi^{*}(c_i,c_j)}{P(C_i) Q(C_j)} \int_{x \in C_i} \int_{y \in C_j} \left( \langle x - c_i, c_i - c_j \rangle \right) \mu(x) \nu(y) dy dx =  & \\
           2 \sum_{i} \sum_{j} \frac{\pi^{*}(c_i,c_j)}{P(C_i) } \int_{x \in C_i} \langle x - c_i, c_i - c_j \rangle \mu(c_j) dx + & \\
           2 \sum_{i} \sum_{j} \frac{\pi^{*}(c_i,c_j)}{P(C_i) } \int_{x \in C_i} \langle x - c_i, c_i - c_j \rangle (\mu(x) - \mu(c_i)) dx 
        \end{split}
    \end{equation} 
    The first integrand above is zero due to symmetry of the linear term about $c_i$. The second term in absolute value is handled using Cauchy-Schwarz and the $\alpha$-H\"older continuity
    of $\mu$, yielding for this term an upper bound of $K h^{1+\alpha}$ where $K$ depens only on $d,C$. Here we have used that $\| c_i - c_j \|_2$ is upper bounded for all $i,j$ by a constant
    depending only on $d$. Hence 
    \[
    \left| 2 \sum_{i} \sum_{j} \frac{\pi^{*}(c_i,c_j)}{P(C_i) Q(C_j)} \int_{x \in C_i} \int_{y \in C_j} \left( \langle x - c_i, c_i - c_j \rangle \right) \mu(x) \nu(y) dy dx \right| \leq K h^{1+\alpha}.
    \]
    By an analgous argument using the $\alpha$-H"older continuity of $\nu$, we have that 
    \[
    \left| 2 \sum_{i} \sum_{j} \frac{\pi^{*}(c_i,c_j)}{P(C_i) Q(C_j)} \int_{x \in C_i} \int_{y \in C_j} \left( \langle c_j - y, c_i - c_j \rangle \right) \mu(x) \nu(y) dy dx \right| \leq K h^{1+\alpha}.
    \]
    Using these and Equation \ref{lowerBoundPart4}, we conclude that for some $K$ depending only $d,C$, 
    \[
    W_2^2(P,Q) \leq W_2^2(G_{h} \# P, G_{h} \# Q) + K h^{1+\alpha}.
    \]
    
\end{proof}

\subsection*{LLM Use Disclosure}
In addition to standard help with formatting or phrasing, we used LLM to search literature.  This included the help in producing a complete related work, and contextualizing results.  Crucially it also helped us identify the \citep{auricchio2018computing} paper which transforms the sketched representation into integral flow problem which can be solved efficiently.  We of course reviewed all literature and techniques before integrating into the paper.


\newpage

\begin{enumerate}

\item {\bf Claims}
    \item[] Question: Do the main claims made in the abstract and introduction accurately reflect the paper's contributions and scope?
    \item[] Answer: \answerYes{} 
    \item[] Guidelines:
    \begin{itemize}
        \item The answer NA means that the abstract and introduction do not include the claims made in the paper.
        \item The abstract and/or introduction should clearly state the claims made, including the contributions made in the paper and important assumptions and limitations. A No or NA answer to this question will not be perceived well by the reviewers. 
        \item The claims made should match theoretical and experimental results, and reflect how much the results can be expected to generalize to other settings. 
        \item It is fine to include aspirational goals as motivation as long as it is clear that these goals are not attained by the paper. 
    \end{itemize}

\item {\bf Limitations}
    \item[] Question: Does the paper discuss the limitations of the work performed by the authors?
    \item[] Answer: \answerYes{} 
    \item[] Justification: We carefully state the conditions of the theorems and lemmas.  We have a separate limitations section.  
    \item[] Guidelines:
    \begin{itemize}
        \item The answer NA means that the paper has no limitation while the answer No means that the paper has limitations, but those are not discussed in the paper. 
        \item The authors are encouraged to create a separate "Limitations" section in their paper.
        \item The paper should point out any strong assumptions and how robust the results are to violations of these assumptions (e.g., independence assumptions, noiseless settings, model well-specification, asymptotic approximations only holding locally). The authors should reflect on how these assumptions might be violated in practice and what the implications would be.
        \item The authors should reflect on the scope of the claims made, e.g., if the approach was only tested on a few datasets or with a few runs. In general, empirical results often depend on implicit assumptions, which should be articulated.
        \item The authors should reflect on the factors that influence the performance of the approach. For example, a facial recognition algorithm may perform poorly when image resolution is low or images are taken in low lighting. Or a speech-to-text system might not be used reliably to provide closed captions for online lectures because it fails to handle technical jargon.
        \item The authors should discuss the computational efficiency of the proposed algorithms and how they scale with dataset size.
        \item If applicable, the authors should discuss possible limitations of their approach to address problems of privacy and fairness.
        \item While the authors might fear that complete honesty about limitations might be used by reviewers as grounds for rejection, a worse outcome might be that reviewers discover limitations that aren't acknowledged in the paper. The authors should use their best judgment and recognize that individual actions in favor of transparency play an important role in developing norms that preserve the integrity of the community. Reviewers will be specifically instructed to not penalize honesty concerning limitations.
    \end{itemize}

\item {\bf Theory assumptions and proofs}
    \item[] Question: For each theoretical result, does the paper provide the full set of assumptions and a complete (and correct) proof?
    \item[] Answer: \answerYes{} 
    \item[] Guidelines:
    \begin{itemize}
        \item The answer NA means that the paper does not include theoretical results. 
        \item All the theorems, formulas, and proofs in the paper should be numbered and cross-referenced.
        \item All assumptions should be clearly stated or referenced in the statement of any theorems.
        \item The proofs can either appear in the main paper or the supplemental material, but if they appear in the supplemental material, the authors are encouraged to provide a short proof sketch to provide intuition. 
        \item Inversely, any informal proof provided in the core of the paper should be complemented by formal proofs provided in appendix or supplemental material.
        \item Theorems and Lemmas that the proof relies upon should be properly referenced. 
    \end{itemize}

    \item {\bf Experimental result reproducibility}
    \item[] Question: Does the paper fully disclose all the information needed to reproduce the main experimental results of the paper to the extent that it affects the main claims and/or conclusions of the paper (regardless of whether the code and data are provided or not)?
    \item[] Answer: \answerNA{} 
    \item[] Justification: The paper does not include experimental results, it is purely theoretical.  
    \item[] Guidelines:
    \begin{itemize}
        \item The answer NA means that the paper does not include experiments.
        \item If the paper includes experiments, a No answer to this question will not be perceived well by the reviewers: Making the paper reproducible is important, regardless of whether the code and data are provided or not.
        \item If the contribution is a dataset and/or model, the authors should describe the steps taken to make their results reproducible or verifiable. 
        \item Depending on the contribution, reproducibility can be accomplished in various ways. For example, if the contribution is a novel architecture, describing the architecture fully might suffice, or if the contribution is a specific model and empirical evaluation, it may be necessary to either make it possible for others to replicate the model with the same dataset, or provide access to the model. In general. releasing code and data is often one good way to accomplish this, but reproducibility can also be provided via detailed instructions for how to replicate the results, access to a hosted model (e.g., in the case of a large language model), releasing of a model checkpoint, or other means that are appropriate to the research performed.
        \item While NeurIPS does not require releasing code, the conference does require all submissions to provide some reasonable avenue for reproducibility, which may depend on the nature of the contribution. For example
        \begin{enumerate}
            \item If the contribution is primarily a new algorithm, the paper should make it clear how to reproduce that algorithm.
            \item If the contribution is primarily a new model architecture, the paper should describe the architecture clearly and fully.
            \item If the contribution is a new model (e.g., a large language model), then there should either be a way to access this model for reproducing the results or a way to reproduce the model (e.g., with an open-source dataset or instructions for how to construct the dataset).
            \item We recognize that reproducibility may be tricky in some cases, in which case authors are welcome to describe the particular way they provide for reproducibility. In the case of closed-source models, it may be that access to the model is limited in some way (e.g., to registered users), but it should be possible for other researchers to have some path to reproducing or verifying the results.
        \end{enumerate}
    \end{itemize}

\item {\bf Open access to data and code}
    \item[] Question: Does the paper provide open access to the data and code, with sufficient instructions to faithfully reproduce the main experimental results, as described in supplemental material?
    \item[] Answer: \answerNA{} 
    \item[] Justification: The paper does not involve experiments, code, or data.  
    \item[] Guidelines:
    \begin{itemize}
        \item The answer NA means that paper does not include experiments requiring code.
        \item Please see the NeurIPS code and data submission guidelines (\url{https://nips.cc/public/guides/CodeSubmissionPolicy}) for more details.
        \item While we encourage the release of code and data, we understand that this might not be possible, so “No” is an acceptable answer. Papers cannot be rejected simply for not including code, unless this is central to the contribution (e.g., for a new open-source benchmark).
        \item The instructions should contain the exact command and environment needed to run to reproduce the results. See the NeurIPS code and data submission guidelines (\url{https://nips.cc/public/guides/CodeSubmissionPolicy}) for more details.
        \item The authors should provide instructions on data access and preparation, including how to access the raw data, preprocessed data, intermediate data, and generated data, etc.
        \item The authors should provide scripts to reproduce all experimental results for the new proposed method and baselines. If only a subset of experiments are reproducible, they should state which ones are omitted from the script and why.
        \item At submission time, to preserve anonymity, the authors should release anonymized versions (if applicable).
        \item Providing as much information as possible in supplemental material (appended to the paper) is recommended, but including URLs to data and code is permitted.
    \end{itemize}

\item {\bf Experimental setting/details}
    \item[] Question: Does the paper specify all the training and test details (e.g., data splits, hyperparameters, how they were chosen, type of optimizer, etc.) necessary to understand the results?
    \item[] Answer: \answerNA{} 
    \item[] Justification: The paper does not include experiments, code, or data.  
    \item[] Guidelines:
    \begin{itemize}
        \item The answer NA means that the paper does not include experiments.
        \item The experimental setting should be presented in the core of the paper to a level of detail that is necessary to appreciate the results and make sense of them.
        \item The full details can be provided either with the code, in appendix, or as supplemental material.
    \end{itemize}

\item {\bf Experiment statistical significance}
    \item[] Question: Does the paper report error bars suitably and correctly defined or other appropriate information about the statistical significance of the experiments?
    \item[] Answer: \answerNA{} 
    \item[] Justification: The paper does not include experiments, code, or data.  
    \item[] Guidelines:
    \begin{itemize}
        \item The answer NA means that the paper does not include experiments.
        \item The authors should answer "Yes" if the results are accompanied by error bars, confidence intervals, or statistical significance tests, at least for the experiments that support the main claims of the paper.
        \item The factors of variability that the error bars are capturing should be clearly stated (for example, train/test split, initialization, random drawing of some parameter, or overall run with given experimental conditions).
        \item The method for calculating the error bars should be explained (closed form formula, call to a library function, bootstrap, etc.)
        \item The assumptions made should be given (e.g., Normally distributed errors).
        \item It should be clear whether the error bar is the standard deviation or the standard error of the mean.
        \item It is OK to report 1-sigma error bars, but one should state it. The authors should preferably report a 2-sigma error bar than state that they have a 96\% CI, if the hypothesis of Normality of errors is not verified.
        \item For asymmetric distributions, the authors should be careful not to show in tables or figures symmetric error bars that would yield results that are out of range (e.g. negative error rates).
        \item If error bars are reported in tables or plots, The authors should explain in the text how they were calculated and reference the corresponding figures or tables in the text.
    \end{itemize}

\item {\bf Experiments compute resources}
    \item[] Question: For each experiment, does the paper provide sufficient information on the computer resources (type of compute workers, memory, time of execution) needed to reproduce the experiments?
    \item[] Answer: \answerNA{} 
    \item[] Justification: The paper does not include experiments, code, or data.  
    \item[] Guidelines:
    \begin{itemize}
        \item The answer NA means that the paper does not include experiments.
        \item The paper should indicate the type of compute workers CPU or GPU, internal cluster, or cloud provider, including relevant memory and storage.
        \item The paper should provide the amount of compute required for each of the individual experimental runs as well as estimate the total compute. 
        \item The paper should disclose whether the full research project required more compute than the experiments reported in the paper (e.g., preliminary or failed experiments that didn't make it into the paper). 
    \end{itemize}
    
\item {\bf Code of ethics}
    \item[] Question: Does the research conducted in the paper conform, in every respect, with the NeurIPS Code of Ethics \url{https://neurips.cc/public/EthicsGuidelines}?
    \item[] Answer: \answerYes{} 
    \item[] Guidelines:
    \begin{itemize}
        \item The answer NA means that the authors have not reviewed the NeurIPS Code of Ethics.
        \item If the authors answer No, they should explain the special circumstances that require a deviation from the Code of Ethics.
        \item The authors should make sure to preserve anonymity (e.g., if there is a special consideration due to laws or regulations in their jurisdiction).
    \end{itemize}

\item {\bf Broader impacts}
    \item[] Question: Does the paper discuss both potential positive societal impacts and negative societal impacts of the work performed?
    \item[] Answer: \answerNA{} 
    \item[] Justification: The paper is theoretical, and it does not have specific societal impact.  
    \item[] Guidelines:
    \begin{itemize}
        \item The answer NA means that there is no societal impact of the work performed.
        \item If the authors answer NA or No, they should explain why their work has no societal impact or why the paper does not address societal impact.
        \item Examples of negative societal impacts include potential malicious or unintended uses (e.g., disinformation, generating fake profiles, surveillance), fairness considerations (e.g., deployment of technologies that could make decisions that unfairly impact specific groups), privacy considerations, and security considerations.
        \item The conference expects that many papers will be foundational research and not tied to particular applications, let alone deployments. However, if there is a direct path to any negative applications, the authors should point it out. For example, it is legitimate to point out that an improvement in the quality of generative models could be used to generate deepfakes for disinformation. On the other hand, it is not needed to point out that a generic algorithm for optimizing neural networks could enable people to train models that generate Deepfakes faster.
        \item The authors should consider possible harms that could arise when the technology is being used as intended and functioning correctly, harms that could arise when the technology is being used as intended but gives incorrect results, and harms following from (intentional or unintentional) misuse of the technology.
        \item If there are negative societal impacts, the authors could also discuss possible mitigation strategies (e.g., gated release of models, providing defenses in addition to attacks, mechanisms for monitoring misuse, mechanisms to monitor how a system learns from feedback over time, improving the efficiency and accessibility of ML).
    \end{itemize}
    
\item {\bf Safeguards}
    \item[] Question: Does the paper describe safeguards that have been put in place for responsible release of data or models that have a high risk for misuse (e.g., pretrained language models, image generators, or scraped datasets)?
    \item[] Answer: \answerNA{} 
    \item[] Justification: The paper is theoretical and has not such risks.  
    \item[] Guidelines:
    \begin{itemize}
        \item The answer NA means that the paper poses no such risks.
        \item Released models that have a high risk for misuse or dual-use should be released with necessary safeguards to allow for controlled use of the model, for example by requiring that users adhere to usage guidelines or restrictions to access the model or implementing safety filters. 
        \item Datasets that have been scraped from the Internet could pose safety risks. The authors should describe how they avoided releasing unsafe images.
        \item We recognize that providing effective safeguards is challenging, and many papers do not require this, but we encourage authors to take this into account and make a best faith effort.
    \end{itemize}

\item {\bf Licenses for existing assets}
    \item[] Question: Are the creators or original owners of assets (e.g., code, data, models), used in the paper, properly credited and are the license and terms of use explicitly mentioned and properly respected?
    \item[] Answer: \answerNA{} 
    \item[] Justification: The paepr does not use existing assets.  
    \item[] Guidelines:
    \begin{itemize}
        \item The answer NA means that the paper does not use existing assets.
        \item The authors should cite the original paper that produced the code package or dataset.
        \item The authors should state which version of the asset is used and, if possible, include a URL.
        \item The name of the license (e.g., CC-BY 4.0) should be included for each asset.
        \item For scraped data from a particular source (e.g., website), the copyright and terms of service of that source should be provided.
        \item If assets are released, the license, copyright information, and terms of use in the package should be provided. For popular datasets, \url{paperswithcode.com/datasets} has curated licenses for some datasets. Their licensing guide can help determine the license of a dataset.
        \item For existing datasets that are re-packaged, both the original license and the license of the derived asset (if it has changed) should be provided.
        \item If this information is not available online, the authors are encouraged to reach out to the asset's creators.
    \end{itemize}

\item {\bf New assets}
    \item[] Question: Are new assets introduced in the paper well documented and is the documentation provided alongside the assets?
    \item[] Answer: \answerNA{} 
    \item[] Justification: The paper does not release new assets.  
    \item[] Guidelines:
    \begin{itemize}
        \item The answer NA means that the paper does not release new assets.
        \item Researchers should communicate the details of the dataset/code/model as part of their submissions via structured templates. This includes details about training, license, limitations, etc. 
        \item The paper should discuss whether and how consent was obtained from people whose asset is used.
        \item At submission time, remember to anonymize your assets (if applicable). You can either create an anonymized URL or include an anonymized zip file.
    \end{itemize}

\item {\bf Crowdsourcing and research with human subjects}
    \item[] Question: For crowdsourcing experiments and research with human subjects, does the paper include the full text of instructions given to participants and screenshots, if applicable, as well as details about compensation (if any)? 
    \item[] Answer: \answerNA{} 
    \item[] Justification: The paper does not involve crowdsourcing or human subjects.  
    \item[] Guidelines:
    \begin{itemize}
        \item The answer NA means that the paper does not involve crowdsourcing nor research with human subjects.
        \item Including this information in the supplemental material is fine, but if the main contribution of the paper involves human subjects, then as much detail as possible should be included in the main paper. 
        \item According to the NeurIPS Code of Ethics, workers involved in data collection, curation, or other labor should be paid at least the minimum wage in the country of the data collector. 
    \end{itemize}

\item {\bf Institutional review board (IRB) approvals or equivalent for research with human subjects}
    \item[] Question: Does the paper describe potential risks incurred by study participants, whether such risks were disclosed to the subjects, and whether Institutional Review Board (IRB) approvals (or an equivalent approval/review based on the requirements of your country or institution) were obtained?
    \item[] Answer: \answerNA{} 
    \item[] Justification: The paper does not involve crowdsourcing or human subjects.  
    \item[] Guidelines:
    \begin{itemize}
        \item The answer NA means that the paper does not involve crowdsourcing nor research with human subjects.
        \item Depending on the country in which research is conducted, IRB approval (or equivalent) may be required for any human subjects research. If you obtained IRB approval, you should clearly state this in the paper. 
        \item We recognize that the procedures for this may vary significantly between institutions and locations, and we expect authors to adhere to the NeurIPS Code of Ethics and the guidelines for their institution. 
        \item For initial submissions, do not include any information that would break anonymity (if applicable), such as the institution conducting the review.
    \end{itemize}

\item {\bf Declaration of LLM usage}
    \item[] Question: Does the paper describe the usage of LLMs if it is an important, original, or non-standard component of the core methods in this research? Note that if the LLM is used only for writing, editing, or formatting purposes and does not impact the core methodology, scientific rigorousness, or originality of the research, declaration is not required.
    \item[] Answer: \answerYes{} 
    \item[] Justification: It is at the end of the Appendix.  
    \item[] Guidelines:
    \begin{itemize}
        \item The answer NA means that the core method development in this research does not involve LLMs as any important, original, or non-standard components.
        \item Please refer to our LLM policy (\url{https://neurips.cc/Conferences/2025/LLM}) for what should or should not be described.
    \end{itemize}

\end{enumerate}

\end{document}